\documentclass[twoside,11pt]{article}
\usepackage[preprint]{jmlr2e}
\usepackage{times}
\usepackage{amsmath,booktabs,microtype}
\newtheorem{assumption}[theorem]{Assumption}

\jmlrheading{0}{2026}{1-1}{8/26}{}{}{Vishal Rajput}
\ShortHeadings{When Can We Trust the Matching Principle?}{Rajput}
\firstpageno{1}

\begin{document}
\title{When Can We Trust the Matching Principle?\\
Robust Deployment Geometry Under Finite-Sample and Model Uncertainty}

\author{\name Vishal Rajput \email vishal.stark42@gmail.com \\
       \addr Independent researcher \\
       Leuven, Belgium}
\editor{}
\maketitle

\begin{abstract}%
Match only geometry you can identify; otherwise spread the penalty.
We quantify that decision by the trust ratio
$\tau=\varepsilon/\gamma$ (estimation uncertainty over spectral separation).
Under the linear--quadratic Matching response, oracle-relative drift between
estimated and oracle projector matching scales as $\tau^2$ for probes in the
chosen top-$r$ deployment subspace---$O(\tau^2)$ in the Davis--Kahan
separation region $\tau<1/2$, with practical usefulness depending on
constants.
Confidence-Calibrated Matching (CCM) turns $\tau$ into a policy---directional
when $\tau$ is small, progressively isotropic when not---with thresholds from
calibration, not from the theorem (match sits in the separation region; soft
is mostly heuristic).
Experiments show both regimes, including UCI HAR embeddings where always-match
is worse than abstain on every cell.
\end{abstract}


\section{Introduction}
\label{sec:intro}

Suppose Matching would be optimal if the deployment direction were known.
We estimate it, regularise aggressively---and do worse than isotropic.
The problem is not whether directional matching is sharp; it is whether the
direction is identifiable.

\begin{quote}
\emph{Match only geometry you can identify; otherwise spread the penalty.}
\end{quote}

That is the whole decision.
Finite samples scramble subspaces when the eigengap is small; a wrong
directional penalty can hurt more than spreading capacity evenly.
Picture the flagship failure: on UCI HAR embeddings, always-matching the
estimated direction is worse than abstaining on $6/6$ cells---not because
Matching is false, but because the estimate was not trustworthy.

\paragraph{The trust ratio.}
Write $\varepsilon$ for a high-probability operator-norm bound on
$\|\widehat\Sigma-\Sigma\|_{\mathrm{op}}$ and
$\gamma=\lambda_r-\lambda_{r+1}$ for the eigengap of the top deployment
subspace.
The scalar that organises the decision is
\[
\tau \;:=\; \frac{\varepsilon}{\gamma}
\;=\;
\frac{\text{estimation uncertainty}}{\text{spectral separation}}.
\]
Small $\tau$: the estimated direction is identifiable; matching it is close to
oracle Matching.
Large $\tau$: do not trust the direction---even-spread remains the default.
Davis--Kahan enters later as a proof tool that converts $\tau$ into a
subspace-error bound; it is not the idea.

\paragraph{What the math says vs.\ what we do.}
Under the LQ Matching response, for probes in the chosen top-$r$ deployment
subspace the oracle-relative absolute drift gap scales as
$|\widehat D-D^\star|\lesssim C\cdot\tau^2$
(Theorem~\ref{thm:drift-gap})---$O(\tau^2)$ in the Davis--Kahan separation
region $\tau<1/2$, with practical usefulness depending on constants.
Confidence-Calibrated Matching (CCM) is a \emph{policy} that uses $\tau$:
directional match, soft blend, or even-spread abstention.
Its thresholds are frozen by a calibration protocol
(\S\ref{sec:ccm})---not a corollary of the theorem---but the match threshold
sits inside that separation region; most of the soft band does not.
Synthetic grids and identifiable positives (Fashion rotations) show the
complementary regime where trust is warranted.

\begin{quote}
\small
\textbf{Scope.}
Oracle-relative regret of estimated projector Matching under the LQ response
(\S\ref{sec:recap})---not a new covariance estimator, not a new Davis--Kahan
theorem, and not global optimality among regularisers.
CCM thresholds are policy, not theorem.
Deep / real blocks are experiments under a specified protocol, not a nonlinear
Matching extension.
\end{quote}

Optional companions (\citet{pmh}, \citet{matching}) motivate why geometry
matters and what coverage means; this paper restates the LQ objects it needs
and is readable alone.

\section{LQ objects (self-contained preliminaries)}
\label{sec:recap}

A cold reader needs only the objects below.
Notation follows~\citet{matching}; every formula used later is stated here.

\paragraph{Displacements and second moments.}
Let $\Delta$ be a label-preserving input displacement with
$\Sigma=\mathbb{E}[\Delta\Delta^\top]\succeq0$.
Large eigenvalues of $\Sigma$ are directions where deployment moves a lot.

\paragraph{Linear--quadratic response.}
For a training penalty $\Sigma'\succeq0$, probe $v$, and weight $\lambda>0$,
\[
w_\lambda(\Sigma';v)=(I+2\lambda\Sigma')^{-1}v,
\qquad
D_\Sigma(w)=w^\top\Sigma w.
\]
As $\lambda\to\infty$, drift $D_\Sigma(w_\lambda)$ vanishes for every
$v\in\mathrm{range}(\Sigma)$ if and only if
$\mathrm{range}(\Sigma')\supseteq\mathrm{range}(\Sigma)$ (cover-or-floor).
Miss any deployment direction and no larger $\lambda$ removes the residual
floor.

\paragraph{Even-spread vs.\ match.}
Even-spread takes $\Sigma'\propto I$: no axis estimate; in this LQ response it
strictly improves on task-only training when axes are unknown.
Matched training takes $\Sigma'$ aligned with $\Sigma$ (or its top subspace):
sharper when axes are known.
Wrong-axis penalties leave a permanent floor.
Operational rule used below: prefer even-spread unless the estimated geometry
is trustworthy enough to match
(\S\ref{sec:ccm}).

\paragraph{What this paper adds.}
The geometry to match is treated as given in the cover question above.
We study $\Sigma'$ as a function of a finite-sample estimate $\widehat\Sigma$ and
ask when that choice stays close to \emph{oracle} matched drift
$D_\Sigma\bigl(w_\lambda(\Pi_W;v)\bigr)$.
The answer is organised by $\tau=\varepsilon/\gamma$
(\S\ref{sec:tau}--\ref{sec:drift}), not by a new cover theorem.

\section{Related work}
\label{sec:related}

We situate by distinction: same tools, different question.

\paragraph{Coverage design.}
\citet{pmh} and \citet{matching} supply the cover question and even-spread
default.
Distinction: we ask when an \emph{estimate} of that geometry is trustworthy
enough to match.

\paragraph{Classical matrix tools.}
Operator-norm covariance bounds~\citep{vershynin}, spiked models~\citep{johnstone2001},
Davis--Kahan / Wedin~\citep{daviskahan,yu2015}, and Ledoit--Wolf
shrinkage~\citep{ledoitwolf} are ingredients or baselines.
Distinction: we package $\varepsilon/\gamma$ as $\tau$ and prove that Matching
\emph{drift} gaps scale as $\tau^2$ on the chosen top-$r$ subspace under the
LQ response---not a new matrix inequality.

\paragraph{Domain shift and DRO.}
CORAL~\citep{sun2016}, DANN~\citep{ganin2016}, DRO /
group-robust training~\citep{bental2009,rahimian2019,sagawa2020}, and
IRM / DG benchmarks~\citep{arjovsky2019,gulrajani2021} address shift with
different objectives and usually without a Matching coverage certificate.
Distinction: CCM is a finite-sample gate on Matching, not a competitor to
those programmes.

\paragraph{Closest concurrent.}
\citet{landers2026} regularise along an estimated drift subspace and report
tolerance to moderate subspace misspecification.
Distinction: their object is pathwise volatility under a frozen predictor;
ours is oracle-relative Matching drift regret as a function of
$\tau=\varepsilon/\gamma$, with CCM as a calibrated match/soft/abstain map.

\paragraph{Selective prediction.}
Selective classification~\citep{geifman2017,elayaniv2010} abstains on uncertain
\emph{labels}.
Distinction: CCM abstains from \emph{directional geometry matching}, falling
back to even-spread rather than refusing a prediction.

\section{Deployment geometry estimation}
\label{sec:estimation}

From displacements $\Delta_1,\ldots,\Delta_n$ form
$\widehat\Sigma=n^{-1}\sum_i\Delta_i\Delta_i^\top$.
Under standard sub-Gaussian assumptions,
$\|\widehat\Sigma-\Sigma\|_{\mathrm{op}}\le\varepsilon(n,\delta)$ with probability
at least $1-\delta$~\citep{vershynin}.
We treat $\varepsilon$ as a cited ingredient and use the proxy
\begin{equation}
\label{eq:eps-proxy}
\varepsilon
\;\approx\;
c\,\|\Sigma\|_{\mathrm{op}}
\Bigl(
\sqrt{\tfrac{d}{n}}
+\sqrt{\tfrac{\log(1/\delta)}{n}}
+\tfrac{d+\log(1/\delta)}{n}
\Bigr),
\end{equation}
with $\delta=0.05$.
Exp.~1 calibrates $c=0.5$ on the synthetic spike grid (full coverage;
$\approx1.5\times$ slack).
On high-dimensional real $\Delta$, ambient-$d$ $\varepsilon$ over-abstains;
Exp.~6 reports a practical numerator alongside ambient-$\tau$ as a limitation
(\S\ref{sec:exp-real}).

\section{The trust ratio}
\label{sec:tau}

Estimation error alone does not decide whether matching is trustworthy.
A large eigengap can absorb moderate covariance error; a tiny gap makes small
$\varepsilon$ fatal.
The decision object is the ratio of those two scales.

\begin{definition}[Trust ratio]
\label{def:tau}
Let $\varepsilon(n,\delta)$ be a high-probability bound on
$\|\widehat\Sigma-\Sigma\|_{\mathrm{op}}$ and
$\gamma:=\lambda_r-\lambda_{r+1}>0$ the eigengap cutting the top-$r$
deployment subspace.
The \emph{trust ratio} is
\begin{equation}
\label{eq:tau}
\tau \;:=\; \frac{\varepsilon(n,\delta)}{\gamma}.
\end{equation}
\end{definition}

Numerator: how uncertain the estimate is.
Denominator: how identifiable the subspace is.
Operational reading---before any perturbation lemma:
$\tau\ll 1$ means the direction is identifiable enough to match;
$\tau\gg 1$ means spread the penalty;
$\tau=O(1)$ is the soft band in between.

\paragraph{Population vs.\ plug-in.}
Theorem~\ref{thm:drift-gap} and Exp.~2 use the population gap $\gamma$.
CCM must run from data alone, so it substitutes the plug-in
$\widehat\gamma=\widehat\lambda_r-\widehat\lambda_{r+1}$
(\S\ref{sec:ccm}).
When $\widehat\gamma$ is unstable the resulting $\tau$ is large and CCM
abstains---the intended behaviour.

\section{What the theorem says}
\label{sec:drift}

Proofs: Appendix~\ref{sec:drift-gap}.

The intellectual object is $\tau=\varepsilon/\gamma$.
Large estimation error relative to the eigengap means: do not trust the
subspace.
Davis--Kahan converts that ratio into a projector gap; the LQ Matching
response then converts the projector gap into an oracle-relative drift gap.
Neither conversion is the idea---$\tau$ is.

\subsection{Perturbation lemma}
\label{sec:dk-lemma}

\begin{lemma}[Subspace error via Davis--Kahan]
\label{lem:dk}
Assume a spiked spectrum with gap $\gamma=\lambda_r-\lambda_{r+1}>0$ and
$\|\widehat\Sigma-\Sigma\|_{\mathrm{op}}\le\varepsilon(n,\delta)$ with
probability at least $1-\delta$.
On the event $2\varepsilon<\gamma$ (equivalently $\tau<1/2$), if
$\widehat W$ spans the top-$r$ eigenspace of $\widehat\Sigma$ and $W$ that of
$\Sigma$, then with probability at least $1-\delta$,
\[
\|\Pi_{\widehat W}-\Pi_W\|_{\mathrm{op}}
\;\le\;
C_{\mathrm{DK}}\,\tau,
\]
for an absolute constant $C_{\mathrm{DK}}$ as in~\citet{yu2015}
(cf.~\citet{daviskahan}).
When $\tau\ge 1/2$ the bound is uninformative: the direction is not
identifiable at this precision.
\end{lemma}

\subsection{Matching consequence}
\label{sec:matching-consequence}

\begin{theorem}[Drift gap scales as $\tau^2$ on the deployment subspace]
\label{thm:drift-gap}
Assume the LQ response of~\S\ref{sec:recap}, the spiked spectrum of
Lemma~\ref{lem:dk}, and $\|\widehat\Sigma-\Sigma\|_{\mathrm{op}}\le\varepsilon$
with probability at least $1-\delta$.
Write $\alpha:=2\lambda/(1+2\lambda)\in(0,1)$.
On the event $\tau<1/2$, for every unit probe
$v\in\mathrm{range}(\Pi_W)$ (the chosen top-$r$ deployment subspace) let
$w_\star=w_\lambda(\Pi_W;v)$ and $w_n=w_\lambda(\Pi_{\widehat W};v)$.
Then, with probability at least $1-\delta$,
\begin{equation}
\label{eq:drift-gap}
\bigl|D_\Sigma(w_n)-D_\Sigma(w_\star)\bigr|
\;\le\;
\bigl(2\alpha+2\alpha^2\bigr)\,C_{\mathrm{DK}}^2\,\|\Sigma\|_{\mathrm{op}}\,\tau^2
\;\lesssim\;
C\,\|\Sigma\|_{\mathrm{op}}\,\tau^2.
\end{equation}
In shorthand: $|\widehat D-D^\star|\lesssim C\cdot\tau^2$
(absolute gap; Appendix~\ref{sec:drift-gap}).
\end{theorem}

\begin{proof}[Proof sketch]
Lemma~\ref{lem:dk} gives
$s:=\|\Pi_{\widehat W}-\Pi_W\|_{\mathrm{op}}\le C_{\mathrm{DK}}\tau$.
For orthogonal projectors the compressed gap satisfies
$\|\Pi_W(I-\Pi_{\widehat W})\Pi_W\|_{\mathrm{op}}\le s^2$.
Expanding the LQ projector response on $v\in\mathrm{range}(\Pi_W)$ then yields an
$O(s^2)=O(\tau^2)$ drift gap (Appendix~\ref{sec:drift-gap}).
Ambient probes outside $\mathrm{range}(\Pi_W)$ retain only the weaker $O(\tau)$
estimate (Remark~\ref{rem:ambient-probe}).
\end{proof}

\paragraph{What this does---and does not---say.}
As $\tau\to 0$, estimated projector matching recovers oracle matched drift at
rate $\tau^2$ for probes in $\mathrm{range}(\Pi_W)$.
Large $\tau$ does not justify trusting $\widehat W$; even-spread
($\Sigma'\propto I$) remains the no-axis default~\citep{matching}.
The bound is oracle-relative regret under correct second-moment
specification---not absolute optimality among regularisers, and not a
numerical recipe for CCM thresholds.
Those thresholds are a calibrated policy (\S\ref{sec:ccm}).
The theorem gives $O(\tau^2)$ scaling in the Davis--Kahan separation region
$\tau<1/2$; practical usefulness depends on constants
(Remark~\ref{rem:bound-live}).
CCM's match threshold sits inside that separation region; most of the soft
band does not.

We state the theorem for \emph{projectors} to isolate subspace error through
$\tau$.
Full $\widehat\Sigma$ matching adds bulk eigenvalues that can shrink task
directions (Appendix~\ref{rem:full-vs-proj}); experiments therefore treat full
$\widehat\Sigma$ as the dangerous naive baseline and the projector as CCM's
match arm.

\section{How CCM uses it}
\label{sec:ccm}

Theorem~\ref{thm:drift-gap} motivates a decision rule: trust $\Pi_{\widehat W}$
when $\tau$ is small; distrust it when $\tau$ is large.
Confidence-Calibrated Matching (CCM) turns that motivation into one penalty
with three arms---match, soft blend, even-spread abstention.

\paragraph{Policy, not corollary.}
CCM is an operational map from $\tau$ to a penalty.
Its thresholds are fixed by the calibration protocol below and then frozen.
They are \emph{not} implied by Theorem~\ref{thm:drift-gap}.
The soft band is a calibrated heuristic; most of it sits outside the proved
Davis--Kahan separation region $\tau<1/2$, where
Theorem~\ref{thm:drift-gap} gives $O(\tau^2)$ scaling (constants decide
numeric usefulness).

\subsection{Inputs and three penalties}
\label{sec:ccm-inputs}

From $\widehat\Sigma$ and $n$, compute $\varepsilon$, the plug-in gap
$\widehat\gamma=\widehat\lambda_r-\widehat\lambda_{r+1}$, and
$\tau=\varepsilon/\widehat\gamma$ ($\tau=\infty$ if $\widehat\gamma\le0$).
This is the operational counterpart of Definition~\ref{def:tau}; the theorem
uses population $\gamma$, the method uses $\widehat\gamma$.
Write $\mu:=\mathrm{tr}(\widehat\Sigma)/d$ and let $\Pi_{\widehat W}$ be the
top-$r$ eigenprojector.

\begin{itemize}
\item \textbf{Full $\widehat\Sigma$:} uses every eigenvalue of the estimate.
  Bulk directions can shrink task signal even when the top subspace is accurate;
  experiments treat this as the dangerous naive baseline.
\item \textbf{Top-$r$ projector $\Pi_{\widehat W}$:} the object controlled by
  Theorem~\ref{thm:drift-gap}; CCM's match arm.
\item \textbf{Trace-scaled even-spread $\mu I$:} no axis estimate; CCM's
  abstention arm.
\end{itemize}

\subsection{Decision rule}
\label{sec:ccm-rule}

Fix $0<\tau_{\mathrm{low}}<\tau_{\mathrm{high}}$ and
\[
\alpha(\tau)
=\mathrm{clip}\!\left(
\frac{\tau-\tau_{\mathrm{low}}}{\tau_{\mathrm{high}}-\tau_{\mathrm{low}}},
\,0,\,1\right).
\]
Then
\begin{equation}
\label{eq:ccm}
\Sigma'_{\mathrm{CCM}}
=
\begin{cases}
\Pi_{\widehat W},
& \tau\le\tau_{\mathrm{low}}\quad\text{(match)},\\[3pt]
\mu I,
& \tau\ge\tau_{\mathrm{high}}\quad\text{(even-spread)},\\[3pt]
\bigl(1-\alpha(\tau)\bigr)\Pi_{\widehat W}+\alpha(\tau)\,\mu I,
& \text{otherwise}\quad\text{(soft)}.
\end{cases}
\end{equation}

\subsection{Threshold-selection protocol}
\label{par:thresholds}

Thresholds are not free per-evaluation knobs and are not implied by
Theorem~\ref{thm:drift-gap}.
They are chosen once by a predetermined procedure, then frozen.

\textit{Synthetic (Exp.~3--5).}
On a development gap$\times n$ grid with the same spike family as Exp.~3 but
seed $11$ (disjoint from the reported Exp.~3 seed), we declare a pair
$(\tau_{\mathrm{low}},\tau_{\mathrm{high}})$ \emph{feasible} if
(i)~match, soft, and even-spread all appear, and
(ii)~on mid-$n$ cells ($n\ge400$) CCM beats both always-full-$\widehat\Sigma$
and always-iso on mean dual cost in at least half the cells.
A coarse search over
$\{0.2,0.35,0.5\}\times\{1.5,2.5,4.0\}$
yields six feasible pairs
(code: \texttt{experiments/calibrate\_thresholds.py};
artifact: \texttt{results/calibration/threshold\_protocol.json}).
We freeze
$(\tau_{\mathrm{low}},\tau_{\mathrm{high}})=(0.35,2.5)$
---feasible, with balanced regime mass---for all reported synthetic tables.
The dual-cost minimiser on that coarse grid is $(0.5,4.0)$; it does not change
the qualitative Exp.~3--5 conclusions (CCM wins mid/high $n$; caution premium at
$n=80$).
Theorem~\ref{thm:drift-gap} gives $O(\tau^2)$ scaling in the separation
region $\tau<1/2$ (Remark~\ref{rem:bound-live}); the frozen match threshold
$\tau_{\mathrm{low}}=0.35$ sits inside that region as a calibrated choice,
not because $C_{\mathrm{DK}}\lesssim 1$ makes every $\tau<1/2$ numerically
live.
Most of the soft interval $(\tau_{\mathrm{low}},\tau_{\mathrm{high}}]$ does
not lie in $\tau<1/2$: the blend there is a calibrated heuristic, not a
certified interpolation of the theorem.

\textit{Real data (Exp.~6).}
Ambient-$d$ $\varepsilon$ forces permanent even-spread, so Exp.~6 switches to
the practical numerator of \S\ref{sec:exp-real} and therefore needs a separate
freeze.
On a Fashion development split (half the loaded images; not the reported test
split), we require an $n$-ladder that progresses out of even-spread toward
soft/match.
Among
$\{(0.5,4),(1,8),(1.5,10),(2,12)\}$,
the pair $(1,8)$ is feasible and is frozen for Office-31 and UCI HAR.
Ambient-$\tau$ is reported alongside to show what the unadjusted proxy would
have done.

Ledoit--Wolf~\citep{ledoitwolf} remains a covariance baseline, not the claim.
Oracle $\Pi_W$ is a synthetic upper reference only.

\section{Experiments}
\label{sec:experiments}

The flagship claim is failure avoidance: when $\tau$ says do not trust the
direction, blindly matching hurts.
Supporting blocks check that $\tau$ tracks subspace error, that CCM wins when
the direction is identifiable, and that a practical $\varepsilon$ opens those
regimes on real features.
Theorem~\ref{thm:drift-gap} is linear--quadratic and oracle-relative; the
blocks below are experiments under a specified protocol---not a proof that CCM
is optimal for every architecture.
Code: \texttt{experiments/}; frozen numbers: \texttt{results/index.json} and
\texttt{results/exp0*/summary.json}.

\subsection{Flagship: always-match hurts when $\tau$ forbids trust}
\label{sec:exp-flagship}

\paragraph{UCI HAR embeddings (Exp.~6, abstention value).}
TCN pooled hidden states ($d=96$).
Deltas $\Delta=h_{\mathrm{aug}}-h_{\mathrm{clean}}$ use label-preserving
sensor nuisance (channel dropout / time-warp style) on the same TCN; geometry
$\widehat\Sigma$ is estimated from those $\Delta$, then CCM chooses the
penalty for a downstream activity probe.
(Checkpoint provenance: Matching release~\citep{matching}, reproducibility
only.)

Practical $\tau$ correlates with subspace error
($\mathrm{corr}(\tau,\sin\Theta)=0.87$), but always-match \emph{hurts} on
$6/6$ cells---nuisance and task directions are entangled in the embedding.
CCM stays soft/even-spread, beats always-match on $6/6$, and beats B0 on
$3/6$ (Figure~\ref{fig:exp06}, bottom).
That is the discovery in operational form: match only geometry you can
identify; otherwise spread the penalty.
A trust gate that only helps when matching helps is incomplete; the HAR cell
shows the complementary failure avoided.

\paragraph{Fashion-MNIST rotation (identifiable positive control).}
Same Exp.~6 protocol on pixel deltas
$\Delta=\mathrm{rotate}_\theta(x)-x$ with $|\theta|\le15^\circ$.
As $n$ grows, CCM regimes progress even-spread$\to$soft$\to$match;
$\mathrm{corr}(\tau,\sin\Theta)=0.92$;
always-match and CCM beat B0 on $6/6$ $n$-cells
(Figure~\ref{fig:exp06}, middle).
This is the clean ``match when identifiable'' complement to HAR.

\paragraph{Office-31 Amazon$\to$DSLR (mid-confidence).}
Frozen ResNet-18 penultimate features ($d=512$); same-class cross-domain
deltas.
$\mathrm{corr}(\tau,\sin\Theta)=0.69$;
always-match beats B0 on $3/4$ cells;
mid-$n$ soft CCM is best among deployable arms ($0.689$ vs.\ always-match
$0.674$).
Ambient-$d$ $\tau$ is $5$--$10\times$ larger and would never leave even-spread:
the practical numerator is doing real work here.

\begin{figure}[t]
  \centering
  \includegraphics[width=\textwidth]{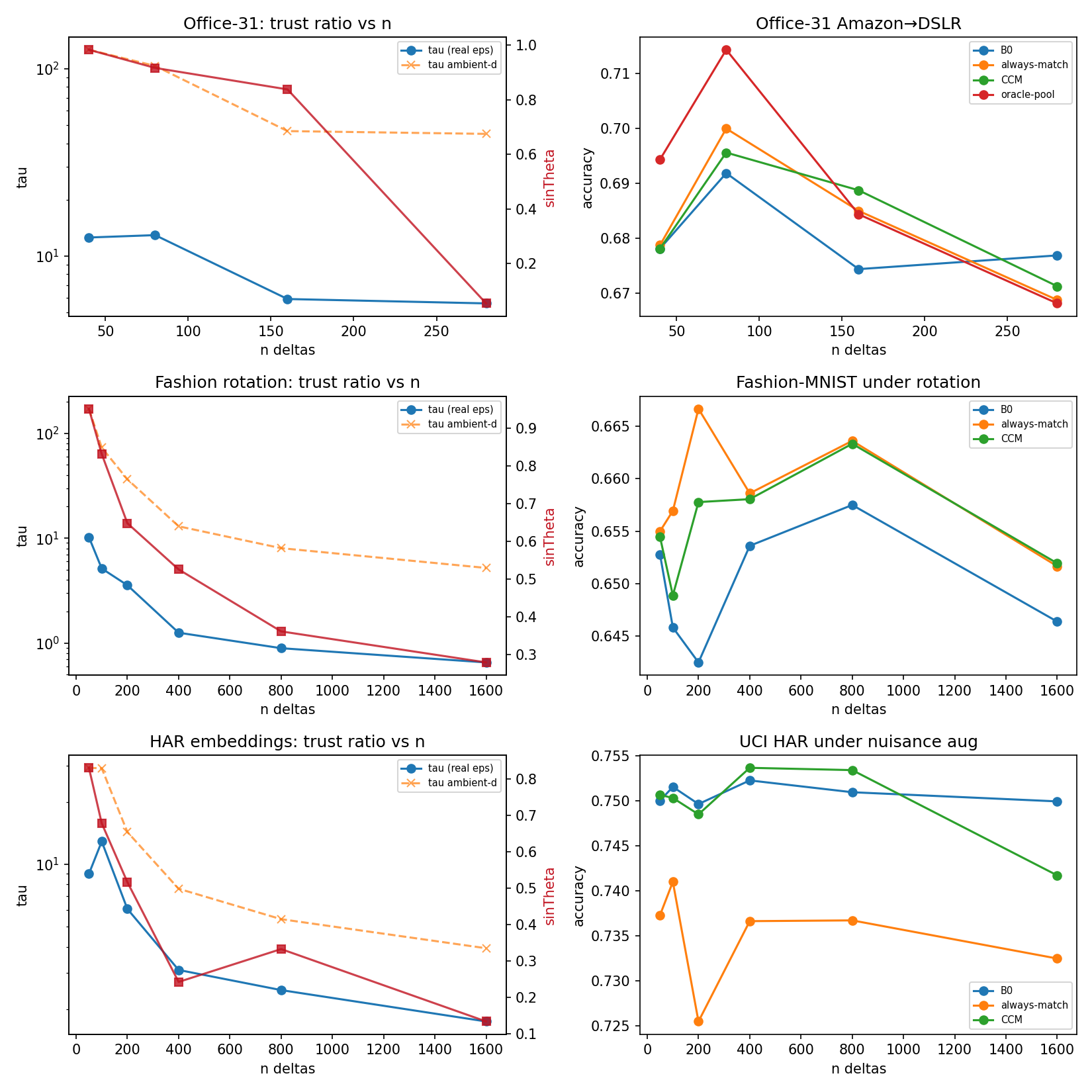}
  \caption{Exp.~6 flagship panel: practical vs.\ ambient trust ratio (left) and
  accuracy arms (right) for Office-31, Fashion rotation, and UCI HAR.
  HAR is the failure-avoidance cell---always-match hurts; CCM abstains.
  Fashion is the identifiable positive control.
  Not a domain-adaptation SOTA table.}
  \label{fig:exp06}
\end{figure}

\paragraph{Shared real-data protocol.}
\label{sec:exp-real}
\begin{itemize}
\item \textbf{Geometry pool.}
  Form label-preserving deltas $\Delta$; estimate $\widehat\Sigma$ from a pool
  of size $n$; choose rank $r$ by largest absolute eigengap among
  $\{4,8,16\}$.
\item \textbf{Practical $\varepsilon$.}
  Ambient-$d$ concentration keeps $\tau$ huge on $d\in\{96,512,784\}$ and
  forces permanent even-spread even when always-match helps.
  Exp.~6 uses
  \[
  \varepsilon
  =
  \min\bigl(
  \varepsilon_{\mathrm{eff}}(n,d_{\mathrm{eff}}=4r),\;
  \|\widehat\Sigma_1-\widehat\Sigma_2\|_{\mathrm{op}}
  \bigr),
  \]
  the split-sample term being an empirical reproducibility radius.
  Ambient-$\tau$ is reported alongside as a limitation, not used for decisions.
\item \textbf{CCM thresholds.}
  Freeze $(1,8)$ from the Fashion development split (\S\ref{sec:ccm}); held
  fixed for Office and HAR---not retuned per dataset.
\item \textbf{Downstream probe.}
  Standardised logistic regression after
  \emph{B0} = no projection;
  \emph{always-match} = project orthogonal to $\widehat W$;
  \emph{CCM} = match / soft / identity per~\eqref{eq:ccm}.
\end{itemize}

\paragraph{What Exp.~6 does not show.}
It does not certify ambient-$d$ concentration on real features; it does not
prove CCM thresholds transfer beyond the Fashion freeze; it does not claim
Matching beats every DA baseline on Office.
It shows that a trust-gated match/abstain rule avoids the always-match failure
mode on HAR and matches when identifiable on Fashion.

\subsection{Supporting synthetic programme (Exp.~1--5)}
\label{sec:exp-synthetic}

Default spike: $d=32$, rank $r=2$, $\lambda=5$, concentration constant $c=0.5$
(Exp.~1 \emph{tests} the proxy~\eqref{eq:eps-proxy}: coverage $1.0$; average
$\varepsilon/\mathrm{err}_{0.95}\approx1.5\times$).
CCM thresholds: frozen synthetic pair $(0.35,2.5)$ from \S\ref{sec:ccm}.

\paragraph{Dual cost.}
Raw mean LQ drift favours isotropic over-shrinkage.
Exp.~3 and~5 use
\[
J \;=\; \frac{D_W}{D_W^{\mathrm{B0}}} \;+\; \bigl(1-\mathrm{ret}_{u\perp W}\bigr),
\]
with residual nuisance drift $D_W$ on the true subspace $W$ and task retention
on $u\perp W$.
Naive baseline: full $\widehat\Sigma$.
CCM match: top-$r$ projector.
Even-spread: $(\mathrm{tr}\,\widehat\Sigma/d)\,I$.

\paragraph{Exp.~1--2 (trust ingredients).}
Exp.~1 confirms that the proxy $\varepsilon$ covers
$\|\widehat\Sigma-\Sigma\|_{\mathrm{op}}$.
Exp.~2 finds $\mathrm{corr}(\tau,\sin\Theta)\approx0.83$ and
$\mathrm{corr}(\tau,\mathrm{drift\text{-}gap})\approx0.91$
(Figure~\ref{fig:exp02}): $\tau$ tracks the quantities in
Theorem~\ref{thm:drift-gap}.
Packaging validation---secondary to the HAR failure-avoidance cell.

\begin{figure}[t]
  \centering
  \includegraphics[width=0.92\textwidth]{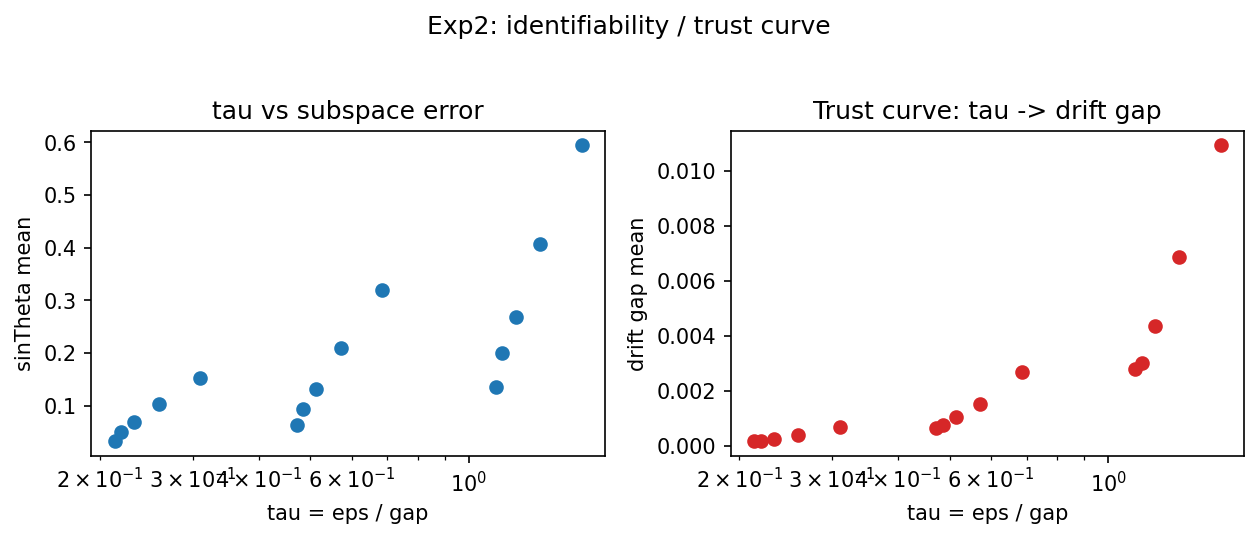}
  \caption{Exp.~2 trust curve on the synthetic spike grid: $\tau$ versus
  subspace error (left) and LQ drift gap to oracle projector match (right).
  Correlations: $0.83$ and $0.91$. Supporting evidence that $\tau$ organises
  the theorem quantities---not the flagship claim.}
  \label{fig:exp02}
\end{figure}

\paragraph{Exp.~3 (CCM decision on spikes).}
On the $3\times3$ gap$\times n$ grid (Table~\ref{tab:exp03},
Figure~\ref{fig:exp03}), regime frequencies are match $33\%$ / soft $59\%$ /
even-spread $8\%$.
CCM beats isotropic on dual cost in $9/9$ cells and beats both pure policies
in $6/9$.
At $n\ge400$, CCM dual cost drops to $0.01$--$0.11$ versus
$0.30$--$0.44$ for full $\widehat\Sigma$.
At $n=80$, CCM pays a caution premium versus noisy $\widehat\Sigma$---the
abstention tradeoff foreshadowing HAR.

\begin{table}[t]
\centering
\caption{Exp.~3 dual cost $J$ (lower better; $d=32$, $30$ reps).
Bold: best among naive / iso / CCM.
Regime $=$ \emph{modal} per-rep CCM decision; cell-mean CCM need not equal the
pure policy named by the mode
(e.g.\ $\gamma{=}0.08$, $n{=}80$: $19/30$ even-spread and $11/30$ soft, hence
CCM $0.653\neq$ iso $0.779$).}
\label{tab:exp03}
\begin{tabular}{rrlcccc}
\toprule
$\gamma$ & $n$ & regime & naive $\widehat\Sigma$ & iso $\mu I$ & CCM & LW \\
\midrule
$0.08$ & $80$   & even & $\mathbf{0.425}$ & $0.779$ & $0.653$ & $0.537$ \\
$0.08$ & $400$  & soft & $0.441$ & $0.774$ & $\mathbf{0.112}$ & $0.471$ \\
$0.08$ & $1600$ & match & $0.439$ & $0.767$ & $\mathbf{0.013}$ & $0.444$ \\
$0.25$ & $80$   & soft & $\mathbf{0.318}$ & $0.759$ & $0.374$ & $0.359$ \\
$0.25$ & $400$  & soft & $0.335$ & $0.755$ & $\mathbf{0.070}$ & $0.341$ \\
$0.25$ & $1600$ & match & $0.331$ & $0.768$ & $\mathbf{0.009}$ & $0.333$ \\
$0.70$ & $80$   & soft & $\mathbf{0.278}$ & $0.750$ & $0.387$ & $0.307$ \\
$0.70$ & $400$  & soft & $0.302$ & $0.749$ & $\mathbf{0.079}$ & $0.307$ \\
$0.70$ & $1600$ & match & $0.308$ & $0.745$ & $\mathbf{0.009}$ & $0.309$ \\
\bottomrule
\end{tabular}
\end{table}

\begin{figure}[t]
  \centering
  \includegraphics[width=0.92\textwidth]{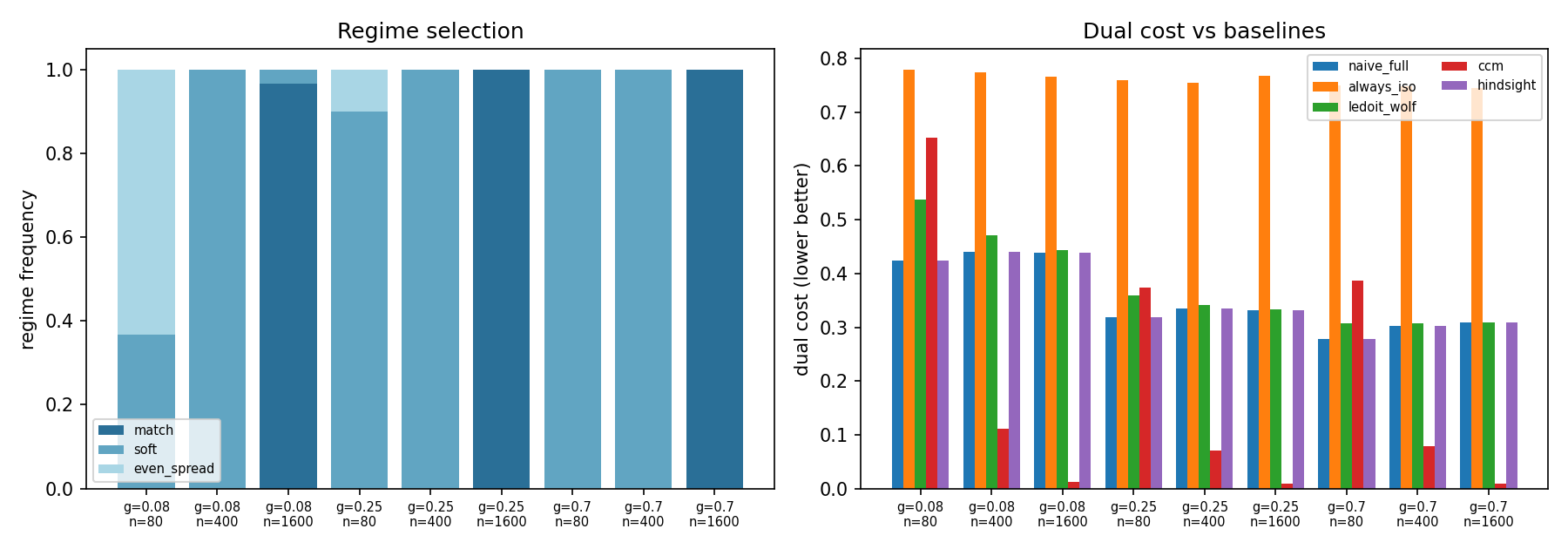}
  \caption{Exp.~3: regime frequencies (left) and dual cost versus fixed policies
  (right) on the gap$\times n$ grid.}
  \label{fig:exp03}
\end{figure}

\paragraph{Exp.~5 (stress).}
The same dual-cost story holds across $d\in\{16,32,64,128\}$ at fixed
$n/d\in\{2.5,12.5,50\}$ (Table~\ref{tab:exp05}):
CCM beats both pure policies in $9/12$ cells and stays within $5\%$ of
hindsight on $84\%$ of reps.

\begin{table}[t]
\centering
\caption{Exp.~5 dual cost at gap $\gamma=0.25$ (means over $20$ reps).
CCM wins whenever $n/d$ is not tiny.}
\label{tab:exp05}
\begin{tabular}{rrlccc}
\toprule
$d$ & $n/d$ & regime & naive & CCM & iso \\
\midrule
$16$  & $2.5$  & soft  & $\mathbf{0.374}$ & $0.562$ & $0.751$ \\
$16$  & $12.5$ & soft  & $0.347$ & $\mathbf{0.122}$ & $0.754$ \\
$16$  & $50$   & match & $0.350$ & $\mathbf{0.009}$ & $0.757$ \\
$32$  & $2.5$  & soft  & $\mathbf{0.312}$ & $0.466$ & $0.765$ \\
$32$  & $12.5$ & soft  & $0.328$ & $\mathbf{0.076}$ & $0.752$ \\
$32$  & $50$   & match & $0.335$ & $\mathbf{0.009}$ & $0.750$ \\
$64$  & $2.5$  & soft  & $\mathbf{0.284}$ & $0.304$ & $0.777$ \\
$64$  & $12.5$ & soft  & $0.295$ & $\mathbf{0.045}$ & $0.781$ \\
$64$  & $50$   & match & $0.296$ & $\mathbf{0.009}$ & $0.770$ \\
$128$ & $2.5$  & soft  & $0.256$ & $\mathbf{0.207}$ & $0.796$ \\
$128$ & $12.5$ & soft  & $0.264$ & $\mathbf{0.031}$ & $0.788$ \\
$128$ & $50$   & match & $0.267$ & $\mathbf{0.009}$ & $0.831$ \\
\bottomrule
\end{tabular}
\end{table}

\paragraph{Exp.~4 (scope; secondary).}
Same-$\Sigma$ / different-$P$ shifts are invisible to $\tau$; wrong-family
geometry breaks naive match while CCM's soft/iso fallback remains safer.
Statistical confidence is not model correctness---controlled misspecification,
not the main product.

\section{Failure boundary}
\label{sec:failure}

Theorem~\ref{thm:drift-gap} and CCM assume the displacements used for
$\widehat\Sigma$ share the deployment $\Sigma$ that matters at test time.
When that assumption fails, confidence is not correctness: a small $\tau$ does
not repair a misspecified geometry family.

\paragraph{Same $\Sigma$, different $P$.}
Exp.~4 (Case~A): $\tau$ can look fine while laws differ in higher-order
structure.
Second-moment confidence is not model correctness
(a natural next question outside LQ risk).

\paragraph{Wrong geometry family.}
Matching $\Sigma_A$ under deployment $\Sigma_B$ fails; CCM soft/even-spread is
the safer fallback (Exp.~4, Case~B)---abstention, not a repair of the wrong
estimand.

\paragraph{Weak gaps and tiny $n$.}
Then $\tau$ is large and CCM abstains by design.

\paragraph{Ambient-$d$ concentration on real $\Delta$.}
Proxy~\eqref{eq:eps-proxy} can over-abstain on high-dimensional features.
Exp.~6 states the practical numerator; sharper real-data bounds remain open.

\paragraph{Out of scope.}
Nonlinear Matching extensions; Dirichlet / heavy nonlinear Matching;
online / LLM Matching; new shrinkage estimators; ambient-$d$ certificates for
Office/Fashion/HAR.

\section{Conclusion}
\label{sec:conclusion}

Match only geometry you can identify; otherwise spread the penalty.
The trust ratio $\tau=\varepsilon/\gamma$ organises that rule: under the LQ
response, oracle-relative projector-matching drift scales as $\tau^2$ for
probes in the chosen top-$r$ deployment subspace
(Theorem~\ref{thm:drift-gap})---$O(\tau^2)$ in the separation region
$\tau<1/2$, with practical usefulness depending on constants.

CCM is a separate calibrated policy from $\tau$ to match / soft / abstain---not
a theorem corollary---with match inside that separation region and soft mostly
outside it.
The flagship evidence is failure avoidance: on UCI HAR embeddings, always-match
hurts on $6/6$ cells while CCM abstains; Fashion shows the identifiable
complement.

\paragraph{Open.}
Sharper real-data $\varepsilon$ with guarantees; adaptive rank with certificates;
trust gates beyond LQ.

\acks{No third-party funding. Optional citations (not required to read this
paper): arXiv:2604.21395 (PMH) and arXiv:2605.22800 (Matching Principle).}

\appendix
\section{Proof of the deployment-drift gap}
\label{sec:drift-gap}

We prove Lemma~\ref{lem:dk} and Theorem~\ref{thm:drift-gap}.
The LQ response is inherited from~\citet{matching}; matrix concentration and
Davis--Kahan are cited tools.
The intellectual object remains $\tau=\varepsilon/\gamma$
(Definition~\ref{def:tau}).

\subsection{Assumptions (restated)}

\begin{assumption}[Label-preserving second moment]
\label{ass:sigma}
Displacements $\Delta$ are label-preserving with finite second moment
$\Sigma:=\mathbb{E}[\Delta\Delta^\top]\succeq0$.
Theory assumes correct second-moment specification (misspecification is Exp.~4).
\end{assumption}

\begin{assumption}[Linear--quadratic response]
\label{ass:lq}
Fix a unit probe $v\in\mathbb{R}^d$ and penalty weight $\lambda>0$.
For $\Sigma'\succeq0$,
\[
w_\lambda(\Sigma';v):=(I+2\lambda\Sigma')^{-1}v,
\qquad
D_\Sigma(w):=w^\top\Sigma w.
\]
\end{assumption}

\begin{assumption}[Spiked deployment spectrum]
\label{ass:spike}
$\Sigma$ has eigenvalues $\lambda_1\ge\cdots\ge\lambda_d\ge0$ with gap
$\gamma:=\lambda_r-\lambda_{r+1}>0$.
Write $W$ for an orthonormal top-$r$ eigenbasis and $\Pi_W:=WW^\top$.
\end{assumption}

\begin{assumption}[Finite-sample geometry]
\label{ass:sample}
I.i.d.\ sub-Gaussian (or bounded) samples yield
$\widehat\Sigma=n^{-1}\sum_i\Delta_i\Delta_i^\top$ with
$\|\widehat\Sigma-\Sigma\|_{\mathrm{op}}\le\varepsilon(n,\delta)$ w.p.\ at least
$1-\delta$ under a standard matrix-concentration bound~\citep{vershynin}.
Experiment~1 calibrates a transparent proxy radius; we do not claim a new
concentration theorem.
\end{assumption}

\subsection{Proof of Lemma~\ref{lem:dk}}

\begin{proof}[Proof of Lemma~\ref{lem:dk}]
Under Assumptions~\ref{ass:spike}--\ref{ass:sample}, write
$E:=\widehat\Sigma-\Sigma$ and assume the concentration event
$\|E\|_{\mathrm{op}}\le\varepsilon(n,\delta)$.
If moreover $2\varepsilon(n,\delta)<\gamma$, the top-$r$ eigenspace remains
separated under the perturbation.
The Davis--Kahan / Wedin form of~\citet{yu2015} then yields
\[
\|\Pi_{\widehat W}-\Pi_W\|_{\mathrm{op}}
\le C_{\mathrm{DK}}\,\frac{\varepsilon(n,\delta)}{\gamma}
= C_{\mathrm{DK}}\,\tau
\]
with probability at least $1-\delta$, for an absolute constant
$C_{\mathrm{DK}}$ (cf.~\citet{daviskahan}).
When $2\varepsilon\ge\gamma$ one has $\tau\ge1/2$ and the bound is
uninformative; CCM treats that regime as non-identifiable and abstains
(\S\ref{sec:ccm}).
\end{proof}

\subsection{Projector response identity}

For the Matching consequence we only need the closed-form projector response
$(I+2\lambda\Pi)^{-1}=I-\alpha\Pi$ with
$\alpha:=2\lambda/(1+2\lambda)$.
(A general resolvent Lipschitz bound exists for non-projector penalties and
recovers a looser factor $4\lambda$; we isolate projectors precisely to remove
that $\lambda$-growth. The resolvent is a calculation, not the trust idea.)

\begin{lemma}[Projector response gap]
\label{lem:projector-gap}
If $A,B$ are orthogonal projectors and
$\alpha:=2\lambda/(1+2\lambda)$, then
\[
(I+2\lambda A)^{-1}-(I+2\lambda B)^{-1}
=\alpha\,(B-A),
\]
so
$\bigl\|(I+2\lambda A)^{-1}-(I+2\lambda B)^{-1}\bigr\|_{\mathrm{op}}
=\alpha\|A-B\|_{\mathrm{op}}$.
\end{lemma}

\begin{proof}
Eigenvalues of $I+2\lambda\Pi$ lie in $\{1,1+2\lambda\}$, hence
$(I+2\lambda\Pi)^{-1}=I-\alpha\Pi$.
Subtracting the two closed forms gives the claim.
\end{proof}

\subsection{Proof of Theorem~\ref{thm:drift-gap}}

\begin{lemma}[Compressed projector gap]
\label{lem:sin2}
If $P,Q$ are orthogonal projectors and $s:=\|P-Q\|_{\mathrm{op}}<1$, then
\[
\|P(I-Q)P\|_{\mathrm{op}}
\;\le\;
s^2.
\]
\end{lemma}

\begin{proof}
In principal angles, the nonzero singular values of $P-Q$ are $\sin\theta_i$ and
those of $P(I-Q)P$ are $\sin^2\theta_i$~\citep{daviskahan}, so
$\|P(I-Q)P\|_{\mathrm{op}}=\max_i\sin^2\theta_i\le s^2$.
\end{proof}

\begin{proof}[Proof of Theorem~\ref{thm:drift-gap}]
On the $1-\delta$ event of Assumption~\ref{ass:sample}, Lemma~\ref{lem:dk} gives
$s:=\|\Pi_{\widehat W}-\Pi_W\|_{\mathrm{op}}\le C_{\mathrm{DK}}\tau$.
Write $P:=\Pi_W$, $Q:=\Pi_{\widehat W}$, and fix a unit $v\in\mathrm{range}(P)$.
Then $w_\star=(I-\alpha P)v=(1-\alpha)v$ and
$D_\Sigma(w_\star)=(1-\alpha)^2\,v^\top\Sigma v$.
Also $w_n=(I-\alpha Q)v$, so
\begin{align*}
D_\Sigma(w_n)
&=
v^\top\Sigma v-2\alpha\,v^\top Q\Sigma v+\alpha^2\,v^\top Q\Sigma Qv,\\
D_\Sigma(w_n)-D_\Sigma(w_\star)
&=
2\alpha\,v^\top(I-Q)\Sigma v
-\alpha^2\,v^\top(\Sigma-Q\Sigma Q)v.
\end{align*}
Because $v\in\mathrm{range}(P)$ and $\Sigma P=P\Sigma$, both $\Sigma v$ and $v$ lie in
$\mathrm{range}(P)$.
For $a,b\in\mathrm{range}(P)$,
$a^\top(I-Q)b=a^\top P(I-Q)Pb$, so Lemma~\ref{lem:sin2} yields
\[
\bigl|v^\top(I-Q)\Sigma v\bigr|
\;\le\;
s^2\|\Sigma\|_{\mathrm{op}}\|v\|_2^2.
\]
Moreover
$\Sigma-Q\Sigma Q=(I-Q)\Sigma+Q\Sigma(I-Q)$, hence
\begin{align*}
\bigl|v^\top(\Sigma-Q\Sigma Q)v\bigr|
&\le
\bigl|v^\top(I-Q)\Sigma v\bigr|
+
\bigl|(Qv)^\top\Sigma(I-Q)v\bigr| \\
&\le
s^2\|\Sigma\|_{\mathrm{op}}\|v\|_2^2
+
\|\Sigma\|_{\mathrm{op}}\|(Q-P)v\|_2\|(I-Q)v\|_2
\\
&\le
2s^2\|\Sigma\|_{\mathrm{op}}\|v\|_2^2,
\end{align*}
where $\|(I-Q)v\|_2=\|(I-Q)Pv\|_2\le s\|v\|_2$ and
$Qv=v+(Q-P)v$.
The triangle inequality then yields the absolute gap
\begin{align*}
\bigl|D_\Sigma(w_n)-D_\Sigma(w_\star)\bigr|
&\le
2\alpha\,\bigl|v^\top(I-Q)\Sigma v\bigr|
+\alpha^2\,\bigl|v^\top(\Sigma-Q\Sigma Q)v\bigr| \\
&\le
\bigl(2\alpha+2\alpha^2\bigr)\|\Sigma\|_{\mathrm{op}}\,s^2
\;\le\;
\bigl(2\alpha+2\alpha^2\bigr)C_{\mathrm{DK}}^2\|\Sigma\|_{\mathrm{op}}\,\tau^2,
\end{align*}
which is~\eqref{eq:drift-gap}.
\end{proof}

\begin{remark}[Ambient probes: $O(\tau)$ only]
\label{rem:ambient-probe}
If $v\notin\mathrm{range}(\Pi_W)$, the bulk term
$v^\top Q(I-\Pi_W)\Sigma(I-\Pi_W)v$ is only $O(s)$ in general.
The crude estimate
$|D_\Sigma(w_n)-D_\Sigma(w_\star)|
\le 2\alpha C_{\mathrm{DK}}\|\Sigma\|_{\mathrm{op}}\tau\|v\|_2^2$
still holds for every unit $v$ (expand
$|(\Delta w)^\top\Sigma(w_n+w_\star)|$ with $\|\Delta w\|_2\le\alpha s\|v\|_2$).
Theorem~\ref{thm:drift-gap} concerns the chosen top-$r$ deployment subspace
$\mathrm{range}(\Pi_W)$---CCM's projector match arm.
Equality $\mathrm{range}(\Sigma)=\mathrm{range}(\Pi_W)$ holds only under exact
rank-$r$ ($\lambda_{r+1}=0$); if $\lambda_{r+1}>0$ the Matching cover game is
larger, and probes outside $\mathrm{range}(\Pi_W)$ retain only the $O(\tau)$
ambient estimate.
\end{remark}

\begin{remark}[Numeric liveness]
\label{rem:bound-live}
Trivially $|D_\Sigma(w_n)-D_\Sigma(w_\star)|\le\|\Sigma\|_{\mathrm{op}}\|v\|_2^2$.
The $\tau^2$ bound beats that when
$(2\alpha+2\alpha^2)C_{\mathrm{DK}}^2\tau^2<1$.
Theorem~\ref{thm:drift-gap} therefore supplies $O(\tau^2)$ scaling in the
Davis--Kahan separation region $\tau<1/2$, but whether that bound is
numerically useful depends on the prefactor constants
($\alpha$, $C_{\mathrm{DK}}$, $\|\Sigma\|_{\mathrm{op}}$)---not on writing
$C_{\mathrm{DK}}\lesssim 1$ and declaring every $\tau<1/2$ ``live.''
CCM's frozen match threshold $\tau_{\mathrm{low}}=0.35$ is a calibrated policy
choice inside the separation region; it is not a corollary of numeric
liveness of the analytic prefactor.
The soft band $(\tau_{\mathrm{low}},\tau_{\mathrm{high}}]$ remains a calibrated
heuristic (\S\ref{sec:ccm}).
Empirically, Exp.~2 drift gaps are far smaller than the analytic prefactor
(scaling sketch, not a calibrated magnitude model).
\end{remark}

\begin{remark}[Full $\widehat\Sigma$ vs.\ projector]
\label{rem:full-vs-proj}
The theorem isolates subspace error through projector penalties.
Matching full $\widehat\Sigma$ adds bulk eigenvalues that can shrink task
directions even when the top subspace is accurate; experiments therefore treat
full $\widehat\Sigma$ as the dangerous naive baseline and the projector as
CCM's match arm.
\end{remark}

\end{document}